%% file: main.tex
\documentclass{article}
\usepackage[numbers]{natbib}

\usepackage[preprint]{neurips_2025}

\usepackage{algorithm}
\usepackage{algorithmic}

\author{
Huanran Li, Jeremy Johnson, Daniel Pimentel-Alarc\'on\\
Department of Electrical Engineering, Mathematics, Biostatistics\\
Wisconsin Institute of Discovery\\
University of Wisconsin-Madison\\
\texttt{\{hli488\}\{jjohnson99\}\{pimentelalar\}@wisc.edu}
}

\title{High Rank Matrix Completion via Grassmannian Proxy Fusion}

\usepackage[utf8]{inputenc} 
\usepackage[T1]{fontenc}    
\usepackage{url}            
\usepackage{booktabs}       
\usepackage{amsfonts}       
\usepackage{nicefrac}       
\usepackage{microtype}      
\usepackage{xcolor}         

\usepackage{graphicx}

\usepackage{amsmath}
\usepackage{amssymb}
\usepackage{mathtools}
\usepackage{amsthm}
\usepackage{dsfont} 
\usepackage[mathscr]{euscript} 

\theoremstyle{plain}
\newtheorem{theorem}{Theorem}[section]

\newcommand{\bs}[1]{\boldsymbol{#1}}

\def \R{\mathbb{R}}
\def \St{\mathbb{S}}
\def \Gr{\mathbb G}

\def \1{{\mathds{1}}}
\def \T{\mathsf{T}}

\def \spn{{\rm span}}
\def \diag{{\rm diag}}

\def \<{\langle}
\def \>{\rangle}

\def \K{{{{\rm K}}}}
\def \n{{{{\rm n}}}}
\def \m{{{{\rm m}}}}
\def \r{{{{\rm r}}}}
\def \p{{{{\rm p}}}}
\def \dist{{{d}}}
\def \lambdaa{{{\lambda}}}

\def \x{{{\bs{{\rm x}}}}}

\def \v{{{\bs{{\rm v}}}}}
\def \w{{{\bs{{\rm w}}}}}

\def \y{{{\bs{{\rm y}}}}}
\def \thetaa{{{\bs{\theta}}}}
\def \xhat{{{\bs{\hat{{\rm x}}}}}}
\def \yhat{{{\bs{\hat{{\rm y}}}}}}

\def \I{{{\bs{{\rm I}}}}}
\def \D{{{\bs{{\rm D}}}}}
\def \X{{{\bs{{\rm X}}}}}

\def \U{{{\bs{{\rm U}}}}}

\def \nablab{{{\bs{\nabla}}}}

\def \Deltab{{{\bs{\Delta}}}}
\def \Gammab{{{\bs{\Gamma}}}}
\def \Epsilonb{{{\bs{{\rm E}}}}}

\def \Thetab{{{\bs{\Theta}}}}
\def \Uhat{{{\bs{\hat{{\rm U}}}}}}

\def \i{{{{\rm i}}}}
\def \j{{{{\rm j}}}}
\def \k{{{{\rm k}}}}
\def \t{{{{\rm t}}}}

\def \o{{{\Omega}}}

\def \sU{{{\mathbb{U}}}}
\def \sX{{{\mathbb{X}}}}

\begin{document}

\maketitle

\begin{abstract}
This paper approaches high-rank matrix completion (HRMC) by filling missing entries in a data matrix where columns lie near a union of subspaces, clustering these columns, and identifying the underlying subspaces. Current methods often lack theoretical support, produce uninterpretable results, and require more samples than theoretically necessary. We propose clustering incomplete vectors by grouping proxy subspaces and minimizing two criteria over the Grassmannian: (a) the chordal distance between each point and its corresponding subspace and (b) the geodesic distances between subspaces of all data points. Experiments on synthetic and real datasets demonstrate that our method performs comparably to leading methods in high sampling rates and significantly better in low sampling rates, thus narrowing the gap to the theoretical sampling limit of HRMC.
\end{abstract}

\section{Introduction}
\label{intro}
\input{0-intro}

\textbf{Paper Organization.} We begin by reviewing previous work and its connection to our model. Next, we introduce our novel GrassFusion model, followed by a discussion of its justification, training dynamics, convergence rate, and computational complexity. Finally, we present experimental results on both synthetically generated and real-world datasets.

\section{Related Work}
\label{relatedWorkSec}
\input{1-background}

\section{Model}
\label{mainSection}
\input{2-model}

\section{Experiments}
\label{experimentSection}

\input{3-Experiment}

\section{Conclusion \& Limitations}
This paper has tackled the challenge of high-rank matrix completion (HRMC) by introducing a novel method that clusters incomplete vectors through proxy subspaces over the Grassmannian. 
Our experiments demonstrate the effectiveness and robustness of GrassFusion across diverse datasets and challenging scenarios. From object tracking in the Hopkins 155 dataset to human activity recognition, handwritten digit clustering, and hyperspectral imaging, GrassFusion consistently achieves competitive performance, often surpassing other state-of-the-art methods, especially under conditions of data with high missing rates. The results highlight GrassFusion's ability to maintain stability across varying sampling rates while delivering low clustering errors. 

One major limitation of our approach is the runtime cost, which averages 10-12 hours per run during experiments. However, this increased computational time is justified by the substantial reduction in clustering error, making our method a valuable tool for HRMC applications where accuracy is paramount.
Overall, this work offers a new perspective on HRMC by leveraging Grassmannian optimization techniques to achieve local convergence guarantees and effectively handle noise. Future research could focus on optimizing runtime and extending the method to handle even more complex datasets and higher-dimensional subspaces, further enhancing its practical utility.

\bibliographystyle{plainnat}
\bibliography{aaai24}

\end{document}

%% file: 0-intro.tex
This paper investigates the task known as high-rank matrix completion (HRMC) \citep{eriksson2012high, elhamifar2016high}, which involves the following objectives: (i) filling in missing entries in a data matrix $\X$ where columns are close to a union of subspaces, (ii) clustering the columns according to these subspaces, and (iii) identifying the underlying subspaces. HRMC is applicable in various domains, such as tracking moving objects in computer vision \citep{kun_huang_minimum_2004, kanatani_motion_2001, xing2024segmentation, yao2024unlabeled}, predicting drug-target interactions for drug discovery \citep{gan2013application, malhat2014clustering}, recognizing user groups in recommender systems \citep{koohi2017new, ullah2014n, zhang2021rp, salamatian2024metascritic}, and inferring network topologies \citep{eriksson2011domainimpute}. Nonetheless, current methods often lack theoretical backing \citep{eriksson2012high, lane2019classifying, yang2015sparse, pimentel2014sample, pimentel2016group, elhamifar2016high, tsakiris2018theoretical}, produce results that are difficult to interpret and thus unsuitable for scientific and security applications \citep{ji2017deep}, and require more samples than what is theoretically necessary \citep{pimentel2016information}.
The main challenge with these limitations is measuring distances (such as Euclidean distances or inner products) between vectors that are only partially observed. This is difficult because calculating distances necessitates overlapping observations, which are less likely in scenarios with sparse sampling \citep{eriksson2012high}. 

To address this issue, we propose a novel approach that clusters incomplete vectors by grouping proxy subspaces, thereby sidestepping the problem of measuring distances between incomplete vectors.
Our approach (see Figure \ref{fig:grassfusion}) involves assigning each point with incomplete data to a unique subspace with complete data and then minimizing two criteria over the Grassmannian: (a) the chordal distance between each point and its corresponding subspace, ensuring the subspace can potentially complete the observed vector, and (b) the geodesic distances between subspaces of all data points, encouraging subspaces of similar points to merge, effectively representing the same space. This optimization takes place over the Grassmannian manifold instead of Euclidean space. Once this optimization is completed, we cluster the proxy subspaces using standard methods like k-means or spectral clustering \citep{bottou1995convergence, von2007tutorial}. This results in a clustering of the incomplete data (goal (ii)). Subsequently, we can fill in the missing entries (goal (i)) using any conventional low-rank matrix completion algorithm \citep{recht2011simpler}. Finally, after clustering and completing the data, the underlying subspaces can be easily identified (goal (iii)) using singular value decomposition. 

\begin{figure*}
    \centering
    \includegraphics[width=0.8\linewidth, clip, trim = {0 4.2cm 2.2cm 0}]{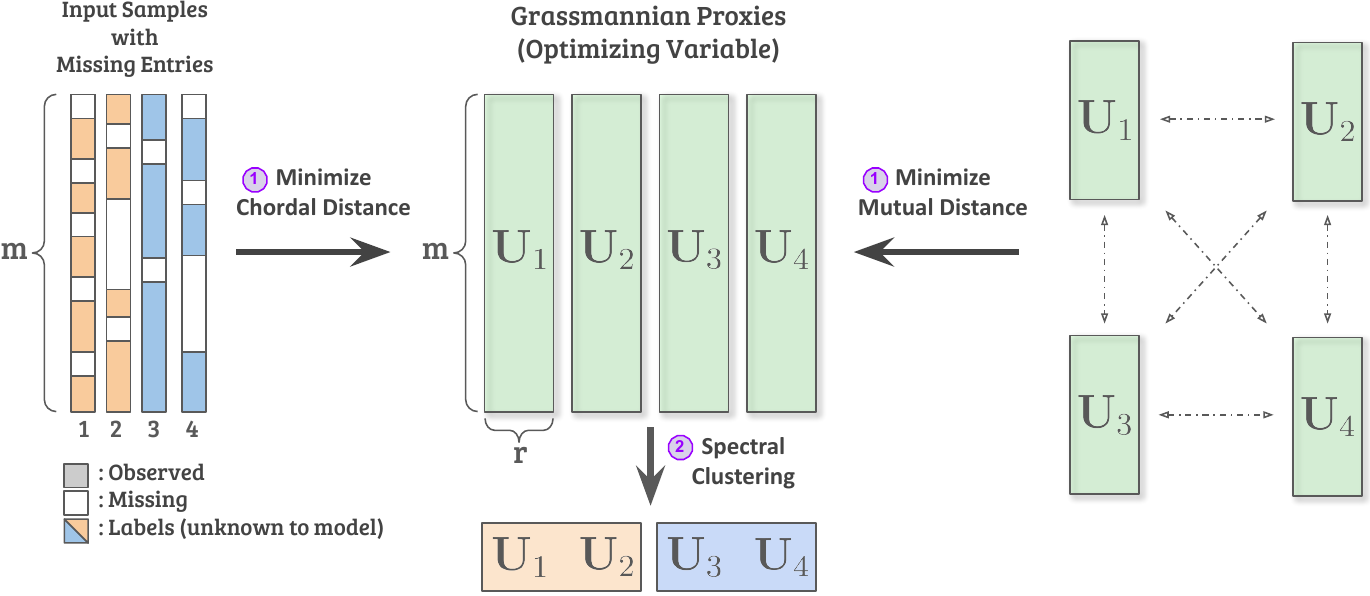}
    \caption{GrassFusion: Assigning each point with incomplete data to a unique subspace with complete data and then minimizing two criteria over the Grassmannian: (a) the chordal distance between each point and its corresponding subspace, ensuring the subspace can potentially complete the observed vector, and (b) the geodesic distances between subspaces of all data points, encouraging subspaces of similar points to merge, effectively representing the same space. Once this optimization is completed, we cluster the proxy subspaces using standard methods like spectral clustering.}
    \label{fig:grassfusion}
\end{figure*}

Transitioning to the Grassmannian framework enables us to utilize established Riemannian manifold optimization techniques and local convergence assurances, which are typically the most robust guarantees available for non-convex problems of this nature \citep{pimentel2014sample, pimentel2016group}. Furthermore, our method does not need prior knowledge of the number of subspaces, is inherently capable of handling noise, and requires only an upper bound on the dimensions of the subspaces. We validate our approach with experiments on both synthetic and real datasets. These experiments demonstrate that even in this initial straightforward version, (a) our method performs on par with the current leading methods in straightforward scenarios (high sampling rates), and (b) our method significantly surpasses the leading methods in challenging scenarios (low sampling rates), thereby narrowing the gap to the fundamental sampling limit of HRMC (see Figure \ref{synthetic_fig}).

%% file: 1-background.tex
HRMC can be viewed as a simultaneous extension of low-rank matrix completion (LRMC) and subspace clustering (SC). In LRMC, the objective is to recover missing entries of a data matrix whose columns reside in a single subspace (hence, low-rank) \citep{recht2011simpler}. HRMC extends LRMC to scenarios where data columns lie in a union of subspaces \citep{eriksson2012high}. Similarly, SC aims to cluster data columns that lie in a union of subspaces \citep{vidal2011subspace}. HRMC generalizes SC to situations where the data is incomplete.

Both LRMC and SC have been extensively studied in recent years, leading to numerous algorithms and guarantees for various conditions, encompassing diverse noise distributions \citep{mcrae2021low}, data distributions \citep{fan2021shrinkage, qu2015subspace}, privacy constraints \citep{wang2015differentially}, outliers \citep{huang2021robust, peng2017constructing}, coherence assumptions \citep{chen2014coherent}, affinity learning \citep{tang2018learning}, sampling schemes \citep{balzano2010online, pimentel2015adaptive}, and more \citep{vidal2005generalized, agarwal2004k, tipping1999probabilistic, derksen2007segmentation}. Thus, for HRMC, if the data could be completed, an SC algorithm could be employed to identify clusters and the underlying union. Conversely, if one could cluster the incomplete data, an LRMC algorithm could be applied to each cluster to fill in the missing entries and identify each subspace.

\textbf{HRMC vs LRMC.} In a high-rank matrix completion (HRMC) problem, if the number of underlying subspaces, denoted by $\K$, and the maximum dimension of these subspaces, denoted by $\r$, are both low, one might consider approaching HRMC as a low-rank matrix completion (LRMC) problem. In this scenario, a single subspace containing all columns of $\X$ would have a dimension no greater than $\r' := \r \cdot \K$. However, this approach ignores the union structure inherent in the data, thus requiring more observed entries to complete $\X$. This can be understood by noting that each column must have more observed entries than the rank of the subspace it belongs to \citep{pimentel2016information}. Consequently, even if $\r'$ is sufficiently low, using LRMC would necessitate $\K$ times more observations than HRMC. This is particularly problematic in fields like Metagenomics or Drug Discovery, where data is extremely sparse and expensive to obtain. Moreover, $\r'$ might be too large to feasibly apply LRMC.

\textbf{HRMC vs SC.} Similarly, for a high-rank matrix completion (HRMC) problem, one common strategy is to first fill in the missing entries using naive methods (such as zeros, means, or low-rank matrix completion) before applying any suitable full-data clustering method. There is a wide range of subspace clustering (SC) theories and algorithms that ensure perfect clustering under reasonable conditions (e.g., adequate sampling and subspace separation) \citep{vidal2005generalized, agarwal2004k, tipping1999probabilistic, derksen2007segmentation, li2024group}. Unfortunately, while this naive approach might work when data is missing at a rate inversely proportional to the dimension of the subspaces \citep{tsakiris2018theoretical}, it tends to fail when there is a moderate amount of missing data. This is because data filled in naively no longer conforms to a union of subspaces \citep{elhamifar2016high}.

\textbf{Tailored HRMC algorithms.} Algorithms designed to address the HRMC problem can be categorized into the following subgroups: (1) \emph{neighborhood} methods, which cluster points based on their overlapping coordinates \citep{eriksson2012high}; (2) \emph{alternating} methods, such as EM \citep{pimentel2014sample}, $k$-subspaces \citep{balzano2012k}, group-lasso \citep{pimentel2016group}, S$^3$LR \citep{li2016structured}, or MCOS \citep{li2021matrix}; (3) \emph{liftings}, which leverage the second-order algebraic structure of unions of subspaces \citep{vidal2005generalized, ongie2017algebraic, fan2018non}; and (4) \emph{integer programming} \citep{soniinteger}. Neighborhood methods require either a large number of observations or a super-polynomial number of samples to produce sufficient overlaps. Liftings involve squaring the dimension of an already high-dimensional problem, which significantly limits their practicality. Integer programming approaches are similarly constrained to small datasets. In summary, while considerable research has been devoted to HRMC, existing algorithms have limitations, and their theoretical guarantees remain largely unexplored.

\textbf{The fundamental sampling limit.} Theoretically, the sampling limit of HRMC is the same as that of LRMC, except for a negligible \emph{checksum} constant \citep{pimentel2016information}. Intuitively, this suggests that HRMC should be able to handle nearly the same amount of missing data as LRMC. However, this is not observed in practice. Unlike LRMC algorithms, all current HRMC algorithms fall short of this limit. An important open question is whether it is possible to reach the information-theoretic sampling limit with a polynomial-time algorithm or if there is an inherent statistical-computational trade-off that prevents this. This paper narrows the gap towards the theoretical sampling limit (see Figure \ref{synthetic_fig}), thereby providing insight into this fundamental open question.

\textbf{Our work in context.} Among the methods discussed, the approach of this paper is conceptually closest to \citep{mishra2019riemannian}, which employs a similar Grassmannian optimization model to tackle the single-subspace problem of LRMC. This paper extends these concepts to the more challenging multiple-subspace problem of HRMC while preserving the local convergence guarantees stated in Proposition 5.1 of \citep{mishra2019riemannian}. The primary distinction between \citep{mishra2019riemannian} and our approach is that the former uses a predefined subset of geodesic distances (see equations (17)-(19) in \citep{mishra2019riemannian}) to identify the Grassmannian points that need to be matched. In \citep{mishra2019riemannian}, these geodesic subsets can be selected somewhat arbitrarily because all points in LRMC belong to the same subspace. Consequently, a \emph{gossip} protocol is effective for the simpler LRMC problem. In contrast, HRMC does not permit prior knowledge of which points must be matched.

%% file: 2-model.tex
\textbf{GrassFusion Objective.} Consider vectors $\x_1, \dots, \x_\n \in \R^\m$ that are close to a union of subspaces, each with a dimension no greater than $\r$. Let $\x_\i^\o \in \R^{|\o_\i|}$ represent the observed components of $\x_\i$, indexed by $\o_\i \subset \{1, \dots, \m\}$. We propose associating each observed vector $\x_\i^\o$ with a proxy subspace $\sU_\i := \spn(\U_\i)$. Our objective is to approximate the actual subspace $\sU_\i^\star$ to which $\x_\i$ belongs by (a) ensuring the proxy subspace $\sU_\i$ includes a potential completion of $\x_\i^\o$ and (b) minimizing the distance between proxy subspaces $\sU_\i$ and $\sU_\j$ to reach a consensus. This is formulated through the following optimization problem, where the first term addresses goal (a) and the second term addresses goal (b):
\begin{align}
\label{stiefelEq}
\min_{\U_1,\dots,\U_\n \in \St(\m,\r)} \
\sum_{\i=1}^\n \dist_c^2(\x_\i^\o,\U_\i)
 + 
\frac{\lambda}{2} \sum_{\i,\j=1}^\n \dist_g^2(\U_\i,\U_\j),
\end{align}
where
$$\dist_c^2(\x_\i^\o,\U_\i) := 1-\sigma_1^2(\X_\i^{0\T} \U_\i),$$
$$\dist_g^2(\U_\i,\U_\j) := \sum_{\ell=1}^\r \arccos^2 \sigma_\ell(\U_\i^\T\U_\j).$$
Here, $\lambda \geq 0$ is a regularization parameter, $\sigma_\ell(\bs{\cdot})$ represents the $\ell^{\rm th}$ largest singular value, and $\X_\i^0$ is the orthonormal matrix that spans all possible completions of a non-zero $\x_\i^\o$. The matrix $\X_\i^0$ can be constructed as follows. If $\x_\i^\o = \bs{0}$, then $\X_\i^0 = \I$, the identity matrix. Otherwise, $\X_\i^0$ is an $\m \times (\m-|\o_\i|+1)$ matrix created by normalizing $\x_\i^\o$, filling the unobserved rows with zeros, and concatenating with the $(\m-|\o_\i|)$ canonical vectors corresponding to the unobserved rows of $\x_\i^\o$. For instance, if $\x_\i^\o \neq \bs{0}$ and is observed in the first $|\o_\i|$ rows, then

\begin{align*}
\X_\i^0 \ = \
\left[ \begin{matrix} \\ \\ \\ \\ \\[-5pt] \end{matrix} \right.
\underbrace{
\begin{array}{c|c}
\frac{\x_\i^\o}{\|\x_\i^\o\|} & \hspace{.3cm} {\bs{0}} \hspace{.3cm} \\[10pt]
\hline
& \\[-5pt]
\bs{0} & {\I}\\
& \\[-5pt]
\end{array}}_{\m-|\o_\i|+1}
\left. \begin{matrix} \\ \\ \\ \\ \\[-5pt] \end{matrix} \right]
\begin{matrix}
\left. \begin{matrix} \vspace{.1cm} \\[11pt] \end{matrix} \right\} |\o_\i| \hspace{.8cm} \\
\left. \begin{matrix} \vspace{.1cm} \\[11pt] \end{matrix} \right\} \m-|\o_\i|.
\end{matrix}
\end{align*}

When $\x_\i^\o$ is fully observed, $\X_\i^0$ simplifies to $\x_\i$ normalized. Recall that the Grassmannian $\Gr(\m, \r)$ is a quotient space of the Stiefel manifold $\St(\m, \r)$ by action of the orthogonal group of $\r \times \r$ orthonormal matrices. Since both terms $\dist_c(\x_\i^\o,\U_\i)$ and $\dist_g(\U_\i,\U_\j)$ are invariant under this quotient, the objective function in \eqref{stiefelEq} does not depend on the choice of basis, and descends to a function on the Grassmannian. 

\textbf{The Reason of Choosing Chordal Distance.} The \emph{chordal distance} $d_c(\x_\i^\o, \U_\i)$, as described in \citep{dai2012}, is not a formal distance on the Grassmannian. Instead, it measures how far $\sU_\i$ is from including a possible completion of $\x_\i^\o$. Specifically, $d_c(\x_\i^\o, \U_\i)$ is defined as the cosine of the angle between the closest completion of $\x_\i^\o$ and the $r$-plane $\sU_\i$. If the largest singular value $\sigma_1(\X_\i^{0\T}\U_\i)$ equals 1, then $\sX_\i^0$ and $\sU_\i$ intersect along at least a line, indicating that the proxy subspace $\sU_\i$ contains a potential completion of $\x_\i^\Omega$. Although merely ensuring that $\sU_\i$ contains a possible completion does not distinguish between different possible completions, consensus is achieved as various proxies $\sU_\i$ and $\sU_\j$ are brought closer by the geodesic term. This means that $\sU_\i$ and $\sU_\j$ can only cluster together if both contain possible completions of both $\x_\i^\Omega$ and $\x_\j^\Omega$. The term \emph{chordal distance}, used in this context, is borrowed from \citep{dai2012} and differs from the conventional chordal distance between points on the Grassmannian \citep{conway1996}. Refer to Figure \ref{geodesicFig} for further intuition.

\begin{figure}
\centering
\includegraphics[height=0.12\textwidth]{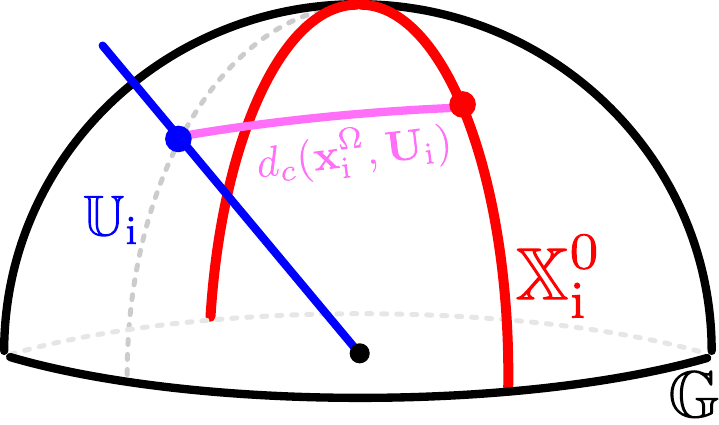}
\hspace{1.2cm}
\includegraphics[height=0.12\textwidth]{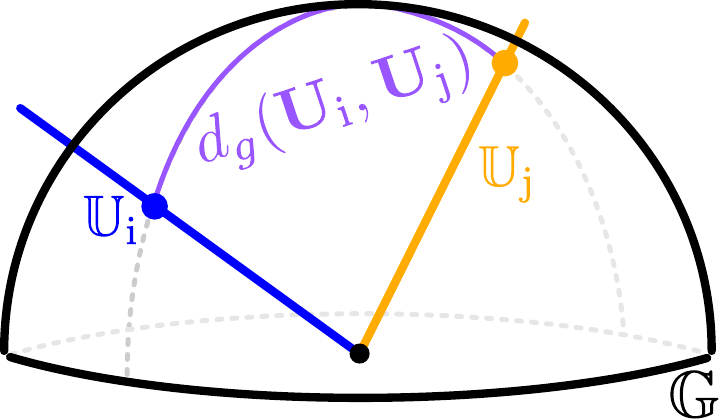}
\hspace{0.75cm}
\includegraphics[height=0.12\textwidth]{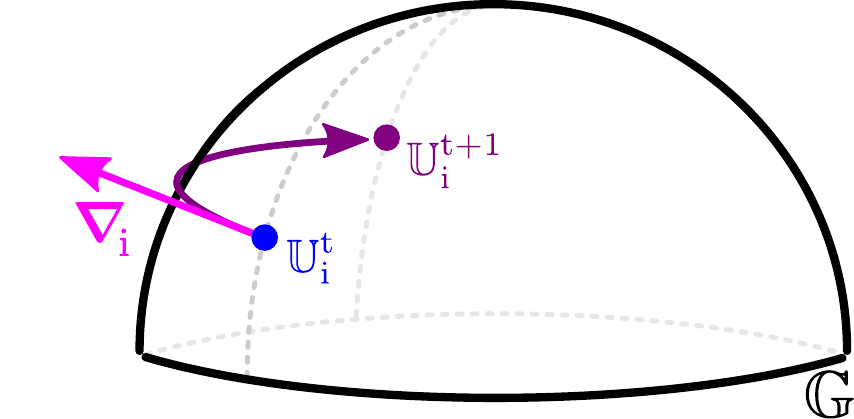}
\caption{The semi-spheres represent the Grassmannian $\Gr(\m, \r)$, where each point $\sU_\i$ denotes a subspace (in the case of $\Gr(3, 1)$, this is represented by the line extending from the origin to $\sU_\i$). \textbf{Left:} The \emph{chordal} distance $\dist_c(\x_\i^\o, \U_\i)$ serves as an informal measure of the distance between the subspace $\sU_\i$ and the incomplete point $\x_\i^\o$. This illustration is intended for intuitive purposes only, as $\sX_\i^0$ may not reside on the same Grassmannian, and the chordal distance should not be equated with geodesic distance. \textbf{Right:} The \emph{geodesic} distance $\dist_g(\U_\i, \U_\j)$ measures the distance over the Grassmannian between $\sU_\i$ and $\sU_\j$. \textbf{Bottom:} The Euclidean gradient vector $\nablab_\i$ deviates from the Grassmann manifold; thus, each geodesic step must be adjusted to account for the curvature of the Grassmannian, as specified by \eqref{stepEq}.}
\label{geodesicFig}
\end{figure}

\textbf{Trade-Off Parameter $\mathbf{\lambda}$ \& Proxy's Rank $\mathbf{r}$.}
The chordal term in \eqref{stiefelEq} ensures that each subspace remains close to its assigned data point, while the geodesic term encourages subspaces from different data points to be close to each other. The balance between these two quantities is controlled by the penalty parameter $\lambdaa \geq 0$. 
For $\lambdaa > 0$, the geodesic term forces subspaces from different data points to move closer together, even if they no longer contain their assigned data points precisely. As $\lambdaa$ increases, the subspaces converge more closely (see Figure \ref{fusionFig}). 
The impact of $\lambdaa$ is ultimately reflected in the distance matrix $\D$, which determines the number of clusters. 
To determine the optimal $\lambdaa$, one can perform a goodness-of-fit test, such as the minimum effective dimension \citep{kun_huang_minimum_2004, vidal2005generalized}, which quantifies the tradeoff between accuracy and degrees of freedom. In practice, we observed that selecting $\lambdaa$ within the range of $[10^{-2}, 10^{-5}]$ yields optimal model performance with only minimal variation. Similarly, we can iteratively increase $\r$ in \eqref{stiefelEq} to identify all data points that lie in 1-dimensional subspaces, then those in 2-dimensional subspaces, and so on (pruning the data at each iteration). This will provide an estimate of the number of subspaces $\K$ and their dimensions.

\begin{figure*}
\centering
\includegraphics[width=0.9\textwidth]{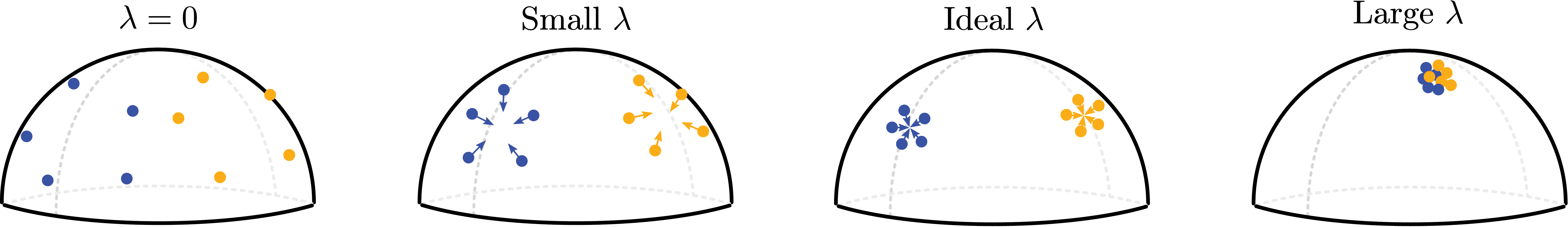}
\caption{$\lambdaa \geq 0$ in \eqref{stiefelEq} regulates how clusters fuse together. If $\lambdaa=0$, each point is assigned to a subspace that exactly contains it (overfitting). The larger $\lambdaa$, the more we penalize subspaces being apart, which results in subspaces getting closer to form fewer clusters. The extreme case $\lambdaa=\infty$ is the special case of PCA and LRMC, where only one subspace is allowed to explain all data.}
\label{fusionFig}
\end{figure*}

\textbf{Proxy Initialization.}
We initialize the proxies such that each contains its assigned observed vector, ensuring that each chordal term is zero. For each $\i$, we first construct an $\m \times \r$ matrix with columns drawn from $\mathcal{N}(0, I_m)$, then replace the first column with $\x_\i^\o$. Next, we orthonormalize this matrix to produce the initial estimate $\U_\i$. By construction, $\U_\i$ includes $\x_\i^\o$, ensuring that $\dist_c(\x_\i^\o, \U_\i) = 0$.

\textbf{Gradient Descent on Grassmannian.} Given the highly non-convex nature of this problem, guaranteeing anything beyond local convergence is challenging. It's important to highlight that the innovation of this work lies in the model description rather than in the optimization process. Numerous standard techniques and libraries \citep{pymanopt} are available for this purpose. We chose to employ standard gradient \citep{edelman_geometry_1998} and line-search \citep{absil2009optimization} methods, which have been successfully applied in other research \citep{mishra2019riemannian, dai2012}. The gradients of \eqref{stiefelEq} with respect to $\U_\i$ over the Grassmannian are given by:

\vspace*{-2.5mm}
\begin{align*}
\nablab \dist^{2}_c(\x_\i^\o,\U_\i) &= -2 \sigma_1( \X_\i^{0\T} \U_\i) (\I-\U_\i \U_\i^\T) \v_\i\w_\i^\T,   \\
\nablab \dist_g^2(\U_\i,\U_\j) &=
\sum_{\ell=1}^r \frac{2\arccos \sigma_\ell(\U_\i^\T\U_\j)}{\sqrt{1-\sigma_\ell^2(\U_\i^\T\U_\j)}}
(\U_\i\U_\i^\T-\I)\v_{\i\j}^\ell\w_{\i\j}^{\ell\T},
\end{align*}
where $\v_\i$ and $\w_\i$ are the leading left and right singular vectors of $\X_\i^0 \X_\i^{0\T} \U_\i$, while $\v_{\i\j}^\ell$ and $\w_{\i\j}^\ell$ are the $\ell^{\rm th}$ left and right singular vectors of $\U_\j\U_\j^\T\U_\i$. For a detailed explanation of how Grassmannian gradients are generally computed, refer to \citep{edelman_geometry_1998}.

Let $\nabla_\i$ represent the gradient of \eqref{stiefelEq}. Recall that $\U_\i - \eta \nablab_\i$ deviates from the Grassmannian for any step size $\eta \neq 0$ (see Figure \ref{geodesicFig}). To account for the manifold's curvature, the update after taking a geodesic step of size $\eta$ over the Grassmannian in the direction of $-\nablab_\i$ is given by equation (2.65) in \citep{edelman_geometry_1998}, which in our context simplifies to:
\begin{align}
\label{stepEq}
\U_\i \ &\gets \ \Big[\U_\i \Epsilonb_\i \ \ \ \ \Gammab_\i \Big]
\left[
\begin{matrix}
\diag \cos (\eta \mathbf{\Upsilon}_\i)  \\
\diag \sin (\eta \mathbf{\Upsilon}_\i)
\end{matrix}
\right]
\Epsilonb_\i^\T,
\end{align}
where $\Gammab_\i \mathbf{\Upsilon}_\i \Epsilonb_\i^\T$ is the compact singular value decomposition of $-\nablab_\i$. 

\textbf{Local Convergence.} In our implementation, we utilize Armijo step sizes, a standard method that provides local convergence guarantees \citep{absil2009optimization}. The step size is defined as $\eta = \beta^\nu \eta_0$, where $\eta_0 > 0$, and $\beta, \gamma \in (0,1)$ are Armijo tuning parameters associated with the initial step size and the granularity of the step search \citep{absil2009optimization}. Here, $\nu$ is the smallest non-negative integer such that
\small
\begin{align*}
\sum_{\i=1}^\n f_\i(\U_\i) - f_\i(R_{\U_\i}(\beta^\nu \eta_0 \nablab_\i))
\ \geq \
-\gamma\sum_{\i=1}^\n\langle \nablab_\i, \beta^\nu \eta_0 (-\nablab_\i)\rangle,
\end{align*}
\normalsize
where $f_\i(\U_\i)$ represents the component of the objective function \eqref{stiefelEq} with $\i$ fixed and $\j$ varying, and $R_{\U_\i}(\Deltab)$ executes the geodesic step as described by \eqref{stepEq} in the direction of $\Deltab$.

We adapt Theorem 4.3.1 and Corollary 4.3.2 from \citep{absil2009optimization} to our context for better demonstration and clarity.

\begin{theorem}
\label{mainThm}
Consider the sequence ${(\U_1,\U_2,\dots,\U_\n)}$ generated by the geodesic steps outlined in equation \eqref{stepEq} with Armijo step sizes $\eta$ as defined above. This sequence will converge to a critical point of \eqref{stiefelEq}.
\end{theorem}

\begin{proof} It is sufficient to demonstrate that the gradient steps in \eqref{stepEq} are a specific instance of the Accelerated Line Search (ALS) algorithm described in \citep{absil2009optimization}, where the product manifold $\Gr^n$ represents the Riemannian manifold $\mathcal{M}$, with its tangent space being the Cartesian product of the tangent spaces of each component $\Gr$.

To clarify, let $T\mathcal{M}$ represent the tangent bundle of $\mathcal{M}$, and let $T_\sU\mathcal{M}$ denote the tangent space of $\mathcal{M}$ at $, \sU \in \mathcal{M}$. Here, $, \sU$ is the tuple $(\U_1,\dots,\U_\n)$, and equation \eqref{stepEq} defines the retraction $R_{\U_\i}$ on each component, so that $R_\sU = (R_{\U_1},\dots, R_{\U_n})$. We can verify this as a valid retraction by recognizing equation \eqref{stepEq} as the exponential map $\text{Exp}: T_\sU\Gr \rightarrow \Gr$, noting that on a Riemannian manifold, the exponential map acts as a retraction. Additionally, the product of exponential maps remains an exponential map \citep{absil2009optimization}.

For the sequence of gradient-related tangent vectors, we take the negative gradient, which is inherently gradient-related. The gradient on the product manifold is the Cartesian product of the gradients on each component manifold, i.e., $\nablab(f) = (\nablab(f_1),\dots, \nablab(f_n))$. The inner product on the tangent space is the sum of the inner products on the component tangent spaces. Therefore, if ${\Deltab_{\i,\t}}$, with $\Deltab_{\i,\t} \in T_{\sU_\t}\mathcal{M}i$, is gradient-related for each $\mathcal{M}\i$, then ${(\Deltab_{1,\t}, \dots, \Deltab_{n,\t})}$ is gradient-related on the product manifold.

Moreover, setting $, \sU_{\t+1} = R_{\sU_\t}(\eta_\t\Deltab_\t)$ satisfies the condition required by the ALS algorithm in \citep{absil2009optimization} with $c=1$. Hence, Theorem \ref{mainThm} follows as a direct result of Theorem 4.3.1 and Corollary 4.3.2 in \citep{absil2009optimization}.
\end{proof}

{\textbf{Computational complexity.}
We note that the primary limitation of our approach is its quadratic complexity concerning the number of samples. Fortunately, subspace clustering offers a straightforward method for sketching both samples and features \citep{traganitis2017sketched}. Specifically, one can solve \eqref{stiefelEq} using a subset of $\n' \leq \n$ columns and a subset of $\m' \leq \m$ rows (e.g., those with the most observations), resulting in improved complexity that is quadratic in $\n'$ instead of $\n$. With the solution to \eqref{stiefelEq}, one can apply a clustering method, a LRMC algorithm, and PCA, as previously described, to produce subspace estimates $\Uhat_1, \dots, \Uhat_{\K'}$, with $\K' \leq \K$. Each of the remaining $\n - \n'$ incomplete data points $\x_\i^\o$ not used to solve \eqref{stiefelEq} and having more than $\r$ observations (a fundamental requirement for subspace clustering \citep{pimentel2016information}) can be assigned to the subspace estimate that yields the largest projection coefficient $\thetaa_\i^\k = (\Uhat{}_\k^{\o\T} \Uhat{}_\k^\o)^{-1} \x_\i^\o$, where $\Uhat{}_\k^\o \in \R^{|\o_\i| \times \r}$ represents the restriction of $\Uhat_\k$ to the observed rows of $\x_\i^\o$ (note that $\Uhat{}_\k^{\o\T} \Uhat{}_\k^\o$ is invertible for almost every rank-$\r$ $\Uhat_\k$ whenever $|\o_\i| > \r$ \citep{pimentel2016information}). If $\x_\i^\o$ is assigned to $\Uhat_\k$, its completion can be straightforwardly estimated as $\xhat_\i = \Uhat_\k \thetaa_\i^\k$. Data points $\x_\i^\o$ that are too far from all subspace estimates (i.e., data points whose coefficients are smaller than a predetermined threshold) can be used to solve \eqref{stiefelEq} again for refined clustering.

%% file: 3-Experiment.tex
In this section, we present a series of experiments conducted on both real and synthetic datasets, specifically the Hopkins155 dataset \citep{tron2007benchmark}, the Smartphone dataset for Human Activity Recognition in Ambient Assisted Living (AAL) \citep{UCI}, Handwritten Digit Images (MNIST) \citep{lecun1998mnist}, and Hyperspectral Imaging (HSI) datasets collected by M. Graña, M.A. Veganzons, and B. Ayerdi. For all experiments, we initialize \eqref{stiefelEq} as described in Model Section, with the value of $\r$ known beforehand. We do not specify $\K$ and make no special adjustments for noise, as our approach does not require it. The solution obtained from \eqref{stiefelEq} is used as input for spectral clustering \citep{von2007tutorial}, although other clustering algorithms such as k-means \citep{bottou1995convergence} or DBSCAN \citep{schubert2017dbscan} could also be used, as discussed in Model Section. Accuracy is measured in terms of clustering error, defined as $\min_M \frac{1}{\n} \sum_{i=1}^{\n} \1_{\{M(\yhat) \neq \y\}}$, where $\1$ is the indicator function, and $M$ maps the estimated cluster labels $\yhat \in \{1,\dots,\hat{\K}\}^\n$ assigned to $\x_1^\o,\dots,\x_\n^\o$ to the true labels $\y \in \{1,\dots,\K\}^\n$. All trials are expected to be completed within a maximum of 12 hours on a single-CPU machine.

\textbf{Baseline comparisons.} A recent survey \citep{lane2019classifying} indicates that most state-of-the-art HRMC algorithms, including MC+SSC \citep{yang2015sparse}, EM \citep{pimentel2014sample}, GSSC \citep{pimentel2016group}, MSC \citep{pimentel2016group}, and $k$-subspaces \citep{balzano2012k}, exhibit similar performance, with different algorithms excelling in specific scenarios based on the subspace vs. ambient dimension gap, the fraction of missing data, and the number of subspaces. Drawing from this survey \citep{lane2019classifying} and others \citep{pimentel2016group}, we selected MC+SSC, EM, GSSC, MSC, ZF+SSC and SSC-EWZF as our baselines. The tuning parameters are manually adjusted to optimize clustering performance according to the training results. These baselines have been shown in \citep{lane2019classifying} to perform nearly identically to $k$-subspaces, and k-GROUSE \citep{balzano2012k} in the scenarios examined in our study.

\begin{figure}
    \centering
    \includegraphics[width=\linewidth]{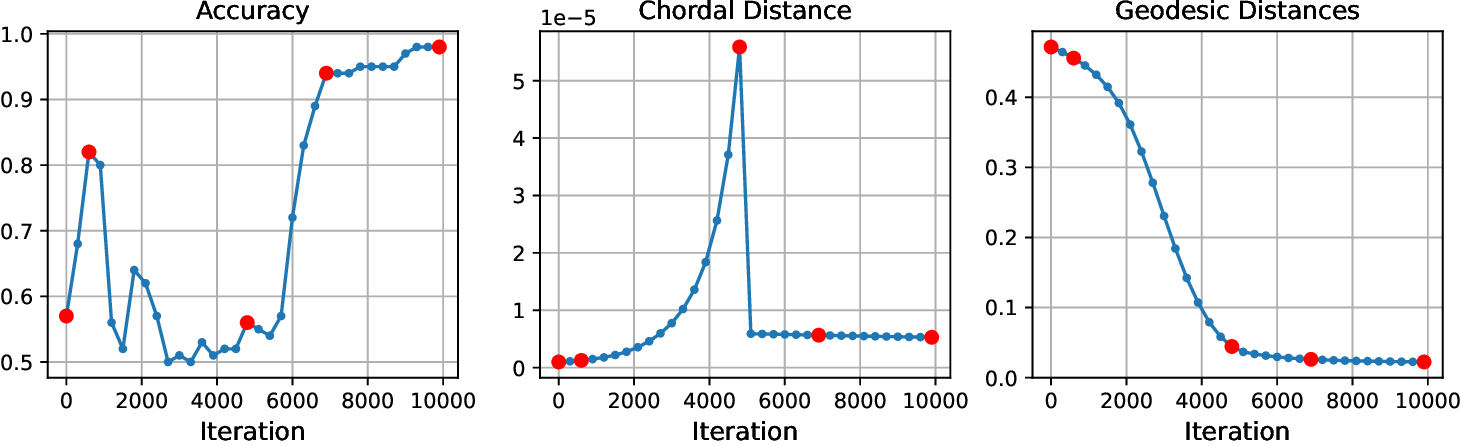}\\ \vspace{0.2cm}
    \includegraphics[width=\linewidth]{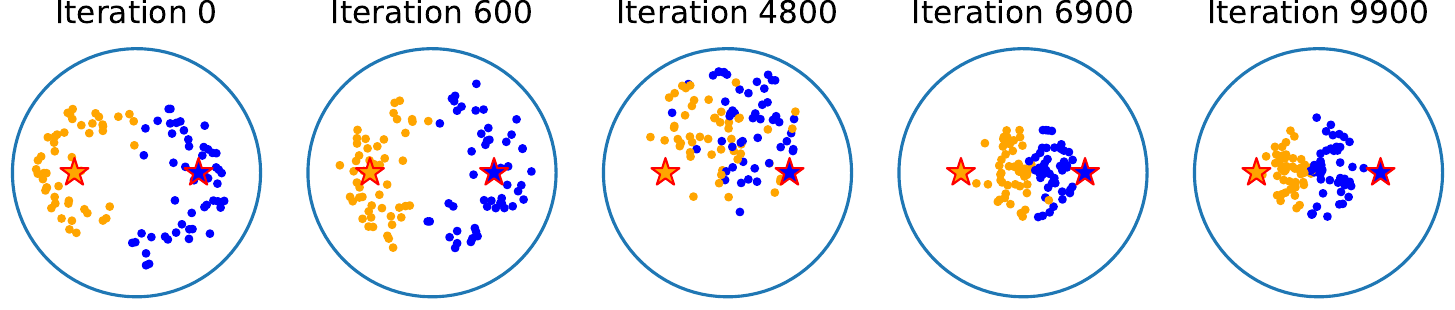}
\caption{Clustering accuracy, sum of chordal distance, and sum of geodesic distance across training iterations. Five key timestamps (marked by red dots at iterations 0, 600, 4800, 6900, and 9900) were selected and visualized using GrassCar\'e \citep{li2024grasscare}, a tool for visualizing subspaces on the Grassmannian. The visualization shows how the subspace proxies (yellow and blue dots) interact with the objective function. The red star indicates the ground-truth subspace used in the initialization.}
    \label{fig:K2}
\end{figure}

\begin{figure*}
\centering
\includegraphics[height = 0.155\linewidth]{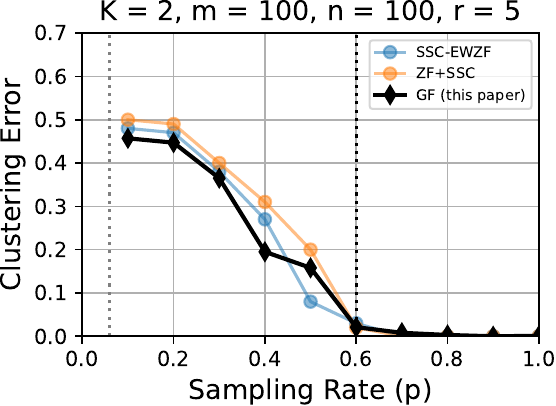}
\includegraphics[height = 0.155\linewidth, clip, trim = {0.5cm 0 0 0}]{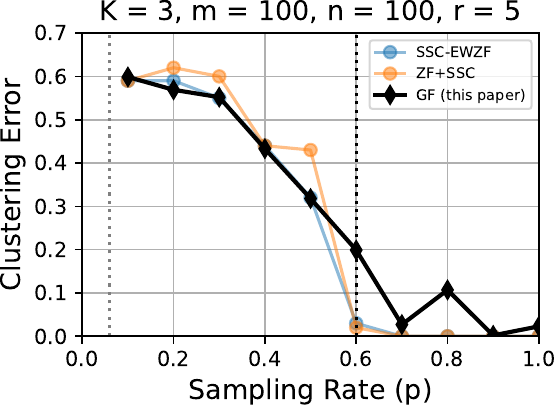}
\includegraphics[height = 0.155\linewidth, clip, trim = {0.5cm 0 0 0}]{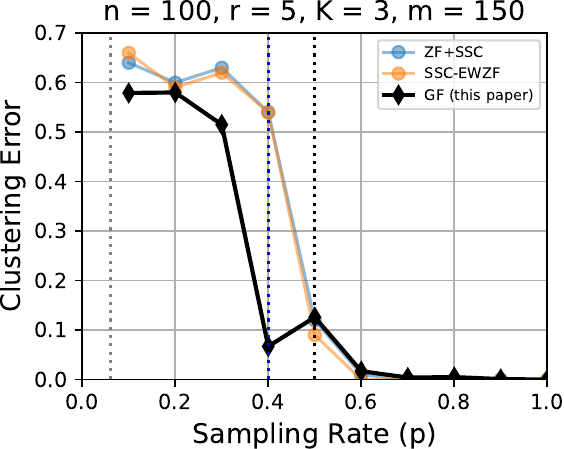}
\includegraphics[height = 0.155\linewidth, clip, trim = {0.5cm 0 0 0}]{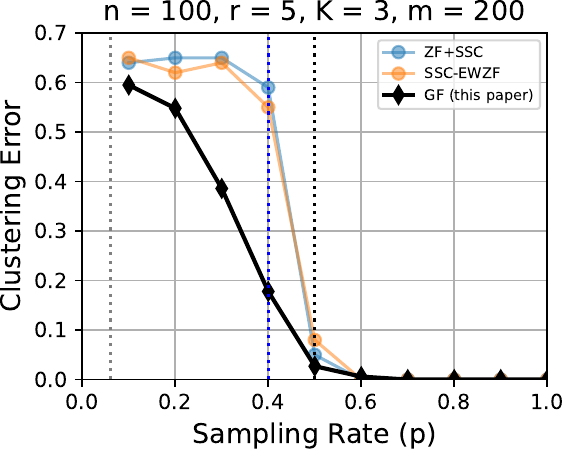}
\includegraphics[height = 0.155\linewidth, clip, trim = {0.5cm 0 0 0}]{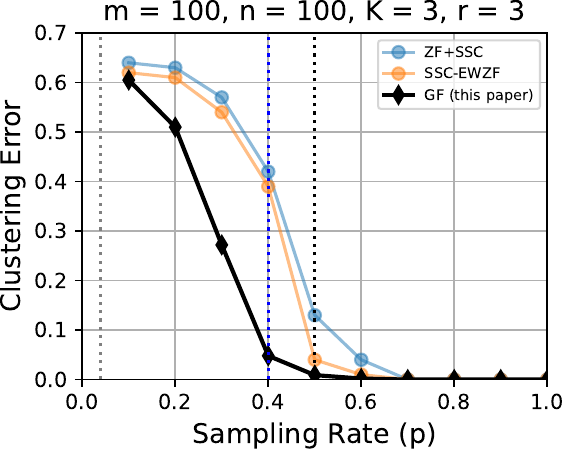}
\caption{Clustering error (average over 10 trials) as a function of sampling rate for different synthetic settings. The left-most vertical line at $\p^\star=(\r+1)/\min(\m,\n)$ represents the information-theoretic sampling limit \citep{pimentel2016information}. That is, HRMC is impossible for any $\p<\p^\star$, and is theoretically possible for any $\p \geq \p^\star$ (for example, with a brute-force combinatorial algorithm). The right-most vertical line indicates the limit of the state-of-the-art. The center vertical line indicates the sampling limit of our approach, which shortens the gap towards the theoretical limit.}
\label{synthetic_fig}
\end{figure*}

\textbf{Synthetic data.}
In all our simulations, we first generate $\K$ matrices $\U_\k^\star \in \R^{\m \times \r}$ with i.i.d.$\mathscr{N}(0,1)$ entries to serve as the bases of the \emph{true} subspaces. For each $\k$, we generate a matrix $\Thetab^\star_\k \in \R^{\r \times \n_\k}$, also with i.i.d.$\mathscr{N}(0,1)$ entries, to serve as the coefficients of the columns in the $\k^{\rm th}$ subspace. We then form $\X$ as the concatenation $[\U_1^\star\Thetab_1^\star, \ \U_2^\star\Thetab_2^\star, \ \dots, \ \U_\K^\star\Thetab_\K^\star]$. To induce missing data, we sample each entry independently with probability $p$. The penalty parameter $\lambda$ was set to $10^{-5}$ for all synthetic experiments.

We first present results for a scenario with 2 classes, 50 samples per cluster, an ambient (feature) dimension of 100, a missing rate of 50\%.
Figure \ref{fig:K2} shows the clustering accuracy, the sum of chordal distances, and the sum of geodesic distances over the training iterations. Additionally, we selected five key timestamps (marked as red dots at iterations 0, 600, 4800, 6900, and 9900) and used GrassCar\'e \citep{li2024grasscare}, a tool specifically designed for visualizing subspaces on the Grassmannian, to illustrate how the subspace proxies (yellow and blue dots) interact with our objective function.
Our observations indicate that GrassFusion initially focuses on minimizing the geodesic distances, pulling all proxies closer together. This initial phase does not necessarily improve accuracy. After 4800 iterations, the algorithm reached low geodesic distances and then experienced a sudden drop in chordal distance. This indicates that the proxies have adjusted to optimally represent the input observations, leading to high clustering accuracy.

We now present results across various settings, including different numbers of clusters, ambient dimensions, subspace ranks, and input missing rates. Given that the clustering performance of our selected baselines has been extensively studied in \citep{lane2019classifying}, and they all perform similarly with synthetically generated data, we chose to focus on SSC-EWZF and ZF+SSC for this experiment. These results are shown in Figure \ref{synthetic_fig}. Our method performs comparably to existing approaches and even excels in scenarios with a subspace rank of 3, especially in low-sampling regimes.

\textbf{Object tracking in Hopkins 155.}
This dataset comprises 155 videos featuring $\K=2$ or $\K=3$ moving objects, including checkerboards, vehicles, and pedestrians. In each video, a collection of $\n$ mark points are tracked across all frames. The locations over time of the $\i^{\rm th}$ point are stacked to form $\x_\i \in \R^\m$, such that points corresponding to the same object lie near a low-dimensional subspace \citep{kanatani_motion_2001} (with $\r$ varying from 1 to 3 depending on the video). For all cases, we fixed the penalty parameter $\lambda$ to $10^{-5}$. To simulate missing data (e.g., due to occlusions), we independently sampled each entry with probability $p$. 

The first plot in Figure \ref{real_fig} presents the clustering results. For each missing rate, the average clustering error across all 155 videos is displayed. GrassFusion generally performs well, with only a slight decline compared to GSSC and MSC.

\begin{figure*}[h]
\centering
\includegraphics[height=0.20\linewidth]{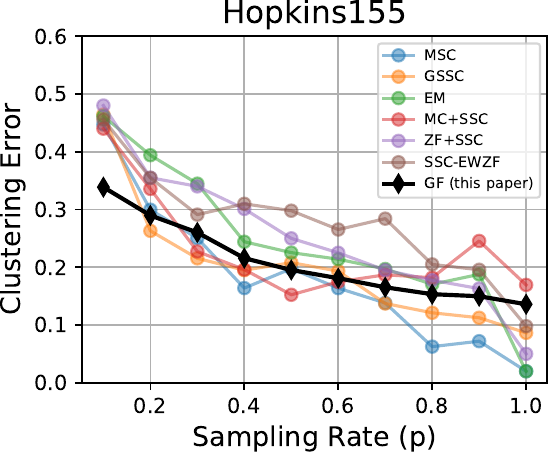}
\includegraphics[height=0.20\linewidth, clip, trim = {0.5cm 0 0 0}]{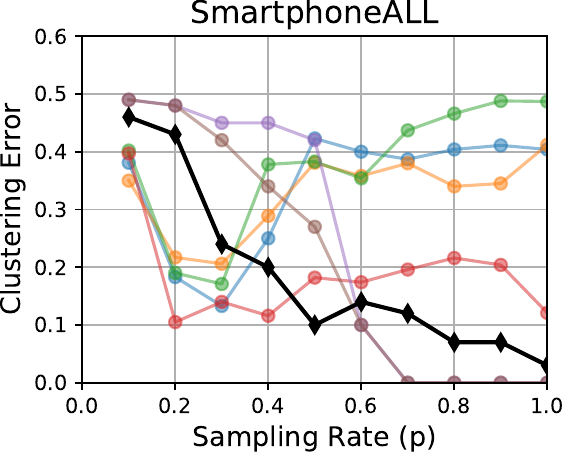}
\includegraphics[height=0.20\linewidth, clip, trim = {0.5cm 0 0 0}]{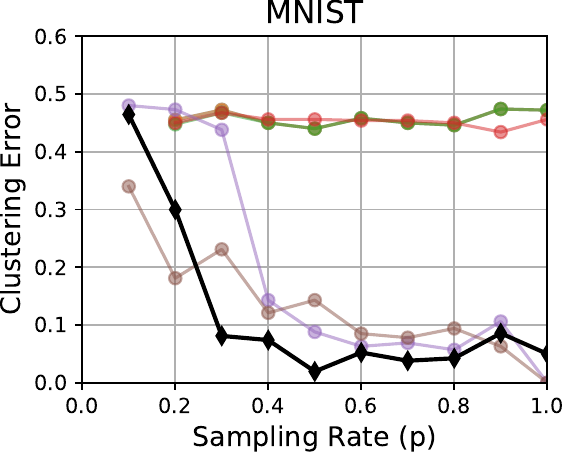}
\includegraphics[height=0.20\linewidth, clip, trim = {0.5cm 0 0 0}]{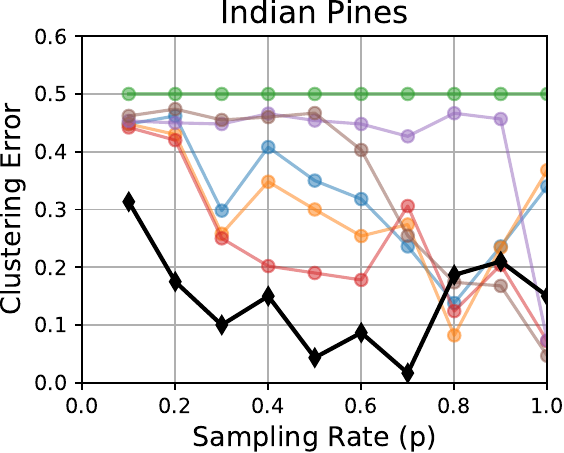} \\ \vspace{0.2cm}
\includegraphics[height=0.20\linewidth]{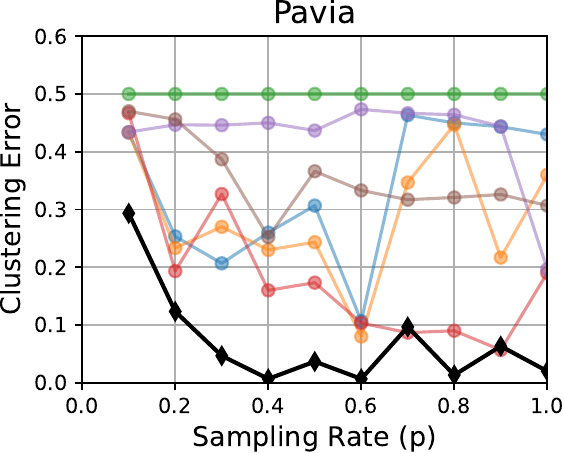} \includegraphics[height=0.20\linewidth, clip, trim = {0.5cm 0 0 0}]{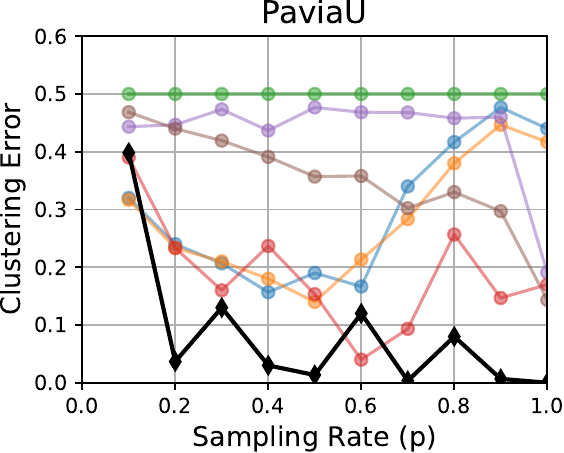} \includegraphics[height=0.20\linewidth, clip, trim = {0.5cm 0 0 0}]{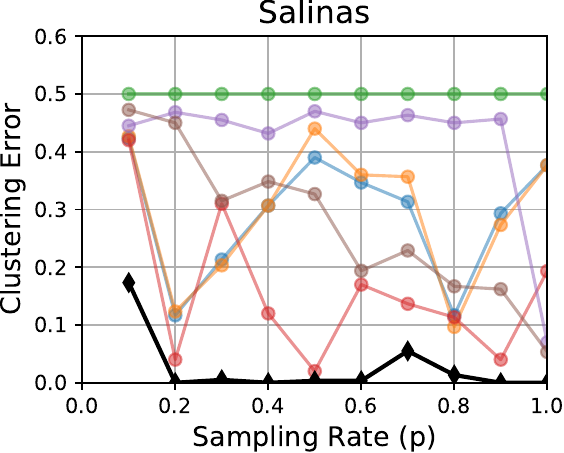} \includegraphics[height=0.20\linewidth, clip, trim = {0.5cm 0 0 0}]{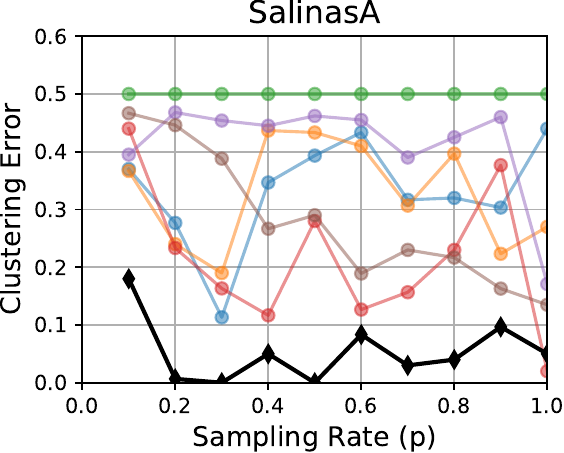} \\

\caption{Average clustering error at various input sampling rates on real-world datasets: Hopkins155 \textbf{(Object Trajectory)}, SmartphoneALL \textbf{(Smartphone Sensor Readings)}, MNIST \textbf{(Handwritten Digit Images)}, and Hyperspectral Imaging datasets \textbf{(Satellite Hyperspectral Sensor Readings)}.}
\label{real_fig}
\end{figure*}

\textbf{Human activity recognition in Smartphone AAL.}
This dataset contains $\n=5744$ instances, each with $\m=561$ features derived from pre-processed accelerometer and gyroscope time series and summary statistics \citep{anguita2013public}. The data corresponds to $\K=2$ activities: walking and other movements, each approximated by a subspace of dimension $\r=4$. Given that the complexity of \eqref{stiefelEq} is quadratic in $\n$, directly solving this dataset would result in unmanageable computational complexity. However, by employing the sketching techniques described in Model Section, it can be solved efficiently. Specifically, we used only $\m'=158$ features (related to the accelerometer's and gyroscope's minimum, maximum, standard deviation, and mean parameters over time) and $\n'=100$ randomly selected samples evenly distributed among classes. The penalty parameter $\lambda$ was fixed at $10^{-5}$ for all cases. 

The results are summarized in the second plot in Figure \ref{real_fig}. For each missing rate, the average clustering error across 20 runs is shown. GF outperforms ZF+SSC and SSC-EWZF in the low-sampling regime and demonstrates greater stability than GSSC, EM, MC+SSC, and MSC in the high-sampling regime. Notably, GF achieves the lowest clustering error at a 50\% sampling rate compared to all other methods.

\textbf{Handwritten Digits MNIST.}
This dataset consists of images from the MNIST dataset \citep{lecun1998gradient}, which includes 10 classes and a total of 60,000 training images. For our experiments, we randomly selected 2 classes, with each class containing 50 images, resulting in a total of 100 images. Prior to the experiments, we vectorized each image and stacked them to form a matrix. We then randomly sampled observed entries to introduce missing data. The penalty parameter $\lambda$ was fixed at $10^{-5}$ throughout the experiments.

The results, summarized in third plot in Figure \ref{real_fig}, show the average clustering error over 20 runs. We observe that SSC-EWZF performs marginally better at 10\% and 20\% sampling rates. However, beyond that, our method, GrassFusion, outperforms all other methods across different sampling rates.

\textbf{Hyperspectral Imaging (HSI) dataset.} 
These datasets consist of images captured from satellites or airborne platforms, focusing on various regions on Earth. Each pixel in these images includes spectral readings from over two hundred different wavelength bands. We selected four of the most widely recognized datasets: Indian Pine, covering $145 \times 145$ pixels with 224 spectral readings per pixel, featuring 16 classes; Pavia University, with $601 \times 601$ pixels and 103 spectral readings per pixel, categorized into 9 land use types; Salinas, depicting the Salinas Valley with $512 \times 217$ pixels and 224 spectral readings per pixel, including 16 classes; and Salinas A, a subset of the Salinas dataset, focusing on $86 \times 83$ pixels with 4 classes representing different stages of romaine lettuce growth. For our experiments, we randomly sampled 2 classes, with each selected class containing 50 spectral images, resulting in a total of 100 images. The penalty parameter $\lambda$ was fixed at $10^{-5}$ throughout the experiments. 

The results are summarized in the plots in Figure \ref{real_fig}. For each sampling rate, the average clustering error over 10 runs is presented. We observe that our model, GrassFusion, consistently performs the best across all sampling rates and is also the most stable algorithm across different scenarios.

%% file: aaai24.bib
@article{pymanopt,
    author = {James Townsend and Niklas Koep and Sebastian Weichwald},
    journal = {Journal of Machine Learning Research},
    number = {137},
    pages = {1–5},
    title = {Pymanopt: A Python Toolbox for Optimization on Manifolds using Automatic Differentiation},
    volume = {17},
    year = {2016}
}

@inproceedings{li2024group,
  title={Group-Sparse Subspace Clustering with Elastic Stars},
  author={Li, Huanran and Pimentel-Alarc{\H{o}}n, Daniel},
  booktitle={2024 IEEE International Symposium on Information Theory (ISIT)},
  pages={1961--1966},
  year={2024},
  organization={IEEE}
}

@inproceedings{salamatian2024metascritic,
  title={metAScritic: Reframing AS-Level Topology Discovery as a Recommendation System},
  author={Salamatian, Loqman and Vermeulen, Kevin and Cunha, Italo and Giotsas, Vasilis and Katz-Bassett, Ethan},
  booktitle={ACM Internet Measurement Conference (IMC’24)},
  year={2024}
}

@article{yao2024unlabeled,
  title={Unlabeled Principal Component Analysis and Matrix Completion},
  author={Yao, Yunzhen and Peng, Liangzu and Tsakiris, Manolis C},
  journal={Journal of Machine Learning Research},
  volume={25},
  number={77},
  pages={1--38},
  year={2024}
}

@article{xing2024segmentation,
  title={Segmentation and Completion of Human Motion Sequence via Temporal Learning of Subspace Variety Model},
  author={Xing, Zheng and Zhao, Weibing},
  journal={IEEE Transactions on Image Processing},
  year={2024},
  publisher={IEEE}
}

@article{ji2017deep,
  title={Deep subspace clustering networks},
  author={Ji, Pan and Zhang, Tong and Li, Hongdong and Salzmann, Mathieu and Reid, Ian},
  journal={Neural information proc. systems},
  volume={30},
  year={2017}
}

@inproceedings{ongie2017algebraic,
  title={Algebraic variety models for high-rank matrix completion},
  author={Ongie, Greg and Willett, Rebecca and Nowak, Robert D and Balzano, Laura},
  booktitle={International Conference on Machine Learning},
  pages={2691--2700},
  year={2017},
  organization={PMLR}
}

@inproceedings{lane2019classifying,
  title={Classifying and comparing approaches to subspace clustering with missing data},
  author={Lane, Connor and Boger, Ron and You, Chong and Tsakiris, Manolis and Haeffele, Benjamin and Vidal, Rene},
  booktitle={Proceedings of the IEEE/CVF International Conference on Computer Vision Workshops},
  pages={0--0},
  year={2019}
}

@article{lecun1998gradient,
  title={Gradient-based learning applied to document recognition},
  author={LeCun, Yann and Bottou, L{\'e}on and Bengio, Yoshua and Haffner, Patrick},
  journal={Proceedings of the IEEE},
  volume={86},
  number={11},
  pages={2278--2324},
  year={1998},
  publisher={Ieee}
}

@article{lecun1998mnist,
  title={The MNIST database of handwritten digits},
  author={LeCun, Yann},
  journal={http://yann. lecun. com/exdb/mnist/},
  year={1998}
}

@book{absil2009optimization,
  title={Optimization algorithms on matrix manifolds},
  author={Absil, P-A and Mahony, Robert and Sepulchre, Rodolphe},
  year={2009},
  publisher={Princeton University Press}
}

@article{mcrae2021low,
  title={Low-rank matrix completion and denoising under Poisson noise},
  author={McRae, Andrew D and Davenport, Mark A},
  journal={Information and Inference: A Journal of the IMA},
  volume={10},
  number={2},
  pages={697--720},
  year={2021},
  publisher={Oxford University Press}
}

@article{huang2021robust,
  title={Robust low-rank matrix completion via an alternating manifold proximal gradient continuation method},
  author={Huang, Minhui and Ma, Shiqian and Lai, Lifeng},
  journal={IEEE Transactions on Signal Processing},
  volume={69},
  pages={2639--2652},
  year={2021},
  publisher={IEEE}
}

@article{fan2021shrinkage,
  title={A shrinkage principle for heavy-tailed data: High-dimensional robust low-rank matrix recovery},
  author={Fan, Jianqing and Wang, Weichen and Zhu, Ziwei},
  journal={Annals of statistics},
  volume={49},
  number={3},
  pages={1239},
  year={2021},
  publisher={NIH Public Access}
}

@inproceedings{balzano2010online,
  title={Online identification and tracking of subspaces from highly incomplete information},
  author={Balzano, Laura and Nowak, Robert and Recht, Benjamin},
  booktitle={Communication, Control, and Computing (Allerton), 2010 48th Annual Allerton Conference on},
  pages={704--711},
  year={2010},
  organization={IEEE}
}

@article{recht2011simpler,
  title={A simpler approach to matrix completion},
  author={Recht, Benjamin},
  journal={Journal of Machine Learning Research},
  volume={12},
  number={Dec},
  pages={3413--3430},
  year={2011}
}

@inproceedings{chen2014coherent,
  title={Coherent matrix completion},
  author={Chen, Yudong and Bhojanapalli, Srinadh and Sanghavi, Sujay and Ward, Rachel},
  booktitle={Int. Conf. on Machine Learning},
  pages={674--682},
  year={2014}
}

@article{mishra2019riemannian,
  title={A Riemannian gossip approach to subspace learning on Grassmann manifold},
  author={Mishra, Bamdev and Kasai, Hiroyuki and Jawanpuria, Pratik and Saroop, Atul},
  journal={Machine Learning},
  volume={108},
  number={10},
  pages={1783--1803},
  year={2019},
  publisher={Springer}
}

@inproceedings{tsakiris2018theoretical,
  title={Theoretical analysis of sparse subspace clustering with missing entries},
  author={Tsakiris, Manolis and Vidal, Ren{\'e}},
  booktitle={International Conference on Machine Learning},
  pages={4975--4984},
  year={2018},
  organization={PMLR}
}

@inproceedings{yang2015sparse,
  title={Sparse subspace clustering with missing entries},
  author={Yang, Congyuan and Robinson, Daniel and Vidal, Rene},
  booktitle={Int. Conf. on Mach. Learning},
  pages={2463--2472},
  year={2015}
}

@article{von2007tutorial,
  title={A tutorial on spectral clustering},
  author={Von Luxburg, Ulrike},
  journal={Statistics and computing},
  volume={17},
  number={4},
  pages={395--416},
  year={2007},
  publisher={Springer}
}

@article{tang2018learning,
  title={Learning a joint affinity graph for multiview subspace clustering},
  author={Tang, Chang and Zhu, Xinzhong and Liu, Xinwang and Li, Miaomiao and Wang, Pichao and Zhang, Changqing and Wang, Lizhe},
  journal={IEEE Transactions on Multimedia},
  volume={21},
  number={7},
  pages={1724--1736},
  year={2018},
  publisher={IEEE}
}

@article{vidal2011subspace,
  title={Subspace clustering},
  author={Vidal, Ren{\'e}},
  journal={IEEE Signal Processing Magazine},
  volume={28},
  number={2},
  pages={52--68},
  year={2011},
  publisher={IEEE}
}

@article{peng2017constructing,
  title={Constructing the L2-graph for robust subspace learning and subspace clustering},
  author={Peng, Xi and Yu, Zhiding and Yi, Zhang and Tang, Huajin},
  journal={IEEE transactions on cybernetics},
  volume={47},
  number={4},
  pages={1053--1066},
  year={2017},
  publisher={IEEE}
}

@inproceedings{qu2015subspace,
  title={Subspace clustering with irrelevant features via robust Dantzig selector},
  author={Qu, Chao and Xu, Huan},
  booktitle={Neural Inf. Proc. Systems},
  pages={757--765},
  year={2015}
}

@inproceedings{wang2015differentially,
  title={Differentially private subspace clustering},
  author={Wang, Yining and Wang, Yu-Xiang and Singh, Aarti},
  booktitle={Neural Information Proc. Systems},
  pages={1000--1008},
  year={2015}
}

@inproceedings{pimentel2014sample,
  title={On the sample complexity of subspace clustering with missing data},
  author={Pimentel, Daniel and Nowak, R and Balzano, Laura},
  booktitle={Statistical Signal Processing (SSP), 2014 IEEE Workshop on},
  pages={280--283},
  year={2014},
  organization={IEEE}
}

@inproceedings{pimentel2015adaptive,
  title={Adaptive strategy for restricted-sampling noisy low-rank matrix completion},
  author={Pimentel-Alarc{\'o}n, Daniel L and Nowak, Robert D},
  booktitle={Computational Advances in Multi-Sensor Adaptive Processing (CAMSAP), 2015 IEEE 6th International Workshop on},
  pages={429--432},
  year={2015},
  organization={IEEE}
}

@inproceedings{pimentel2016group,
  title={Group-sparse subspace clustering with missing data},
  author={Pimentel-Alarc{\'o}n, Daniel and Balzano, Laura and Marcia, Roummel and Nowak, R and Willett, Rebecca},
  booktitle={Statistical Signal Processing Workshop (SSP), 2016 IEEE},
  pages={1--5},
  year={2016},
  organization={IEEE}
}

@inproceedings{pimentel2016information,
  title={The information-theoretic requirements of subspace clustering with missing data},
  author={Pimentel-Alarcon, Daniel and Nowak, Robert},
  booktitle={International Conference on Machine Learning},
  pages={802--810},
  year={2016}
}

@inproceedings{eriksson2011domainimpute,
  title={DomainImpute: Inferring unseen components in the Internet},
  author={Eriksson, Brian and Barford, Paul and Sommers, Joel and Nowak, Robert},
  booktitle={INFOCOM, 2011 Proceedings IEEE},
  pages={171--175},
  year={2011},
  organization={IEEE}
}

@article{li2021matrix,
  title={Matrix completion with column outliers and sparse noise},
  author={Li, Ziheng and Hu, Zhanxuan and Nie, Feiping and Wang, Rong and Li, Xuelong},
  journal={Information Sciences},
  volume={573},
  pages={125--140},
  year={2021},
  publisher={Elsevier}
}

@article{li2016structured,
  title={A structured sparse plus structured low-rank framework for subspace clustering and completion},
  author={Li, Chun-Guang and Vidal, Ren{\'e}},
  journal={IEEE Transactions on Signal Processing},
  volume={64},
  number={24},
  pages={6557--6570},
  year={2016},
  publisher={IEEE}
}

@article{fan2018non,
  title={Non-linear matrix completion},
  author={Fan, Jicong and Chow, Tommy WS},
  journal={Pattern Recognition},
  volume={77},
  pages={378--394},
  year={2018},
  publisher={Elsevier}
}

@article{gossip,
 title={A Riemannian gossip approach to subspace learning on Grassmann manifold},
 author={Mishra, B. and Kasai, H. and Jawanpuria, P. and Saroop, A.},
 journal={Machine Learning},
 volume={108},
 number={10},
 pages={1783--1803},
 year={2019}
}

@inproceedings{tron2007benchmark,
  title={A benchmark for the comparison of 3-d motion segmentation algorithms},
  author={Tron, Roberto and Vidal, Ren{\'e}},
  booktitle={2007 IEEE conference on computer vision and pattern recognition},
  pages={1--8},
  year={2007},
  organization={IEEE}
}

@article{edelman_geometry_1998,
  title={The geometry of algorithms with orthogonality constraints},
  author={Edelman, Alan and Arias, Tom{\'a}s A and Smith, Steven T},
  journal={SIAM journal on Matrix Analysis and Applications},
  volume={20},
  number={2},
  pages={303--353},
  year={1998},
  publisher={SIAM}
}

@inproceedings{malhat2014clustering,
  title={Clustering of chemical data sets for drug discovery},
  author={Malhat, Mohamed G and Mousa, Hamdy M and El-Sisi, Ashraf B},
  booktitle={2014 9th international conference on informatics and systems},
  pages={DEKM--11},
  year={2014},
  organization={IEEE}
}

@inproceedings{kun_huang_minimum_2004,
  title={Minimum effective dimension for mixtures of subspaces: A robust GPCA algorithm and its applications},
  author={Huang, Kun and Ma, Yi and Vidal, Ren{\'e}},
  booktitle={Proc. of IEEE Computer Society Conf. on Computer Vision and Pattern Recognition},
  volume={2},
  pages={II--II},
  year={2004},
  organization={IEEE}
}

@inproceedings{kanatani_motion_2001,
  title={Motion segmentation by subspace separation and model selection},
  author={Kanatani, Ken-ichi},
  booktitle={Proceedings Eighth IEEE International Conference on computer Vision. ICCV 2001},
  volume={2},
  pages={586--591},
  year={2001},
  organization={IEEE}
}

@phdthesis{gan2013application,
  title={The Application of Spectral Clustering in Drug Discovery},
  author={Gan, Sonny},
  year={2013},
  school={University of Sheffield}
}

@article{ullah2014n,
  title={N-screen aware multicriteria hybrid recommender system using weight based subspace clustering},
  author={Ullah, Farman and Sarwar, Ghulam and Lee, Sungchang},
  journal={The Scientific World Journal},
  volume={2014},
  year={2014},
  publisher={Hindawi}
}

@article{zhang2021rp,
  title={RP-LGMC: rating prediction based on local and global information with matrix clustering},
  author={Zhang, Wen and Wang, Qiang and Yoshida, Taketoshi and Li, Jian},
  journal={Computers \& Operations Research},
  volume={129},
  pages={105228},
  year={2021},
  publisher={Elsevier}
}

@article{koohi2017new,
  title={A new method to find neighbor users that improves the performance of collaborative filtering},
  author={Koohi, Hamidreza and Kiani, Kourosh},
  journal={Expert Systems with Applications},
  volume={83},
  pages={30--39},
  year={2017},
  publisher={Elsevier}
}

@inproceedings{eriksson2012high,
  title={High-rank matrix completion},
  author={Eriksson, Brian and Balzano, Laura and Nowak, Robert},
  booktitle={Artificial Intelligence and Statistics},
  pages={373--381},
  year={2012},
  organization={PMLR}
}

@article{soniinteger,
  title={Integer Programming Approaches To Subspace Clustering With Missing Data},
  author={Soni, Akhilesh and Linderoth, Jeff and Luedtke, Jim and Pimentel-Alarc{\'o}n, Daniel}
}

@inproceedings{balzano2012k,
  title={K-subspaces with missing data},
  author={Balzano, Laura and Szlam, Arthur and Recht, Benjamin and Nowak, Robert},
  booktitle={2012 IEEE Statistical Signal Processing Workshop (SSP)},
  pages={612--615},
  year={2012},
  organization={IEEE}
}

@inproceedings{derksen2007segmentation,
  title={Segmentation of multivariate mixed data via lossy coding and compression},
  author={Derksen, Harm and Ma, Yi and Hong, Wei and Wright, John},
  booktitle={Visual Comm. and Image Proc.},
  volume={6508},
  pages={170--181},
  year={2007},
  organization={SPIE}
}

@article{tipping1999probabilistic,
  title={Probabilistic principal component analysis},
  author={Tipping, Michael E and Bishop, Christopher M},
  journal={Journal of the Royal Statistical Society: Series B (Statistical Methodology)},
  volume={61},
  number={3},
  pages={611--622},
  year={1999},
  publisher={Wiley Online Library}
}

@inproceedings{agarwal2004k,
  title={K-means projective clustering},
  author={Agarwal, Pankaj K and Mustafa, Nabil H},
  booktitle={Proceedings of the twenty-third ACM SIGMOD-SIGACT-SIGART symposium on Principles of database systems},
  pages={155--165},
  year={2004}
}

@article{li2024grasscare,
  title={GrassCar{\'e}: Visualizing the Grassmannian on the Poincar{\'e} Disk},
  author={Li, Huanran and Pimentel-Alarc{\'o}n, Daniel},
  journal={SN Computer Science},
  volume={5},
  number={3},
  pages={1--15},
  year={2024},
  publisher={Springer}
}

@article{dai2012,
  title={A Geometric Approach to Low-Rank Matrix Completion},
  author={Wei, Dai and Ely, Kerman and Olgica, Milenkovic},
  journal={IEEE Trans. on Inf. Theory},
  volume={58},
  number={1},
  pages={237--247},
  year={2012}
}

@inproceedings{anguita2013public,
  title={A public domain dataset for human activity recognition using smartphones.},
  author={Anguita, Davide and Ghio, Alessandro and Oneto, Luca and Parra, Xavier and Reyes-Ortiz, Jorge Luis},
  booktitle={Esann},
  volume={3},
  pages={3},
  year={2013}
}

@misc{UCI ,
author = {Anguita, Davide and Ghio, Alessandro and Oneto, Luca and Parra, Xavier and Reyes-Ortiz, Jorge Luis},
year = "2013",
title = "{UCI} Machine Learning Repository",
institution = "University of California, Irvine, School of Information and Computer Sciences" }

@article{elhamifar2016high,
  title={High-rank matrix completion and clustering under self-expressive models},
  author={Elhamifar, Ehsan},
  journal={Neural Inf. Proc. Systems},
  volume={29},
  pages={73--81},
  year={2016}
}

@inproceedings{bottou1995convergence,
  title={Convergence properties of the k-means algorithms},
  author={Bottou, Leon and Bengio, Yoshua},
  booktitle={Neural information proc. systems},
  pages={585--592},
  year={1995}
}

@article{schubert2017dbscan,
  title={DBSCAN revisited: why and how you should (still) use DBSCAN},
  author={Schubert, Erich and Sander, J{\"o}rg and Ester, Martin and Kriegel, Hans Peter and Xu, Xiaowei},
  journal={ACM Transactions on Database Systems (TODS)},
  volume={42},
  number={3},
  pages={1--21},
  year={2017},
  publisher={ACM New York, NY, USA}
}

@article{vidal2005generalized,
  title={Generalized principal component analysis (GPCA)},
  author={Vidal, Rene and Ma, Yi and Sastry, Shankar},
  journal={IEEE transactions on pattern analysis and machine intelligence},
  volume={27},
  number={12},
  pages={1945--1959},
  year={2005},
  publisher={IEEE}
}

@article{traganitis2017sketched,
  title={Sketched subspace clustering},
  author={Traganitis, Panagiotis A and Giannakis, Georgios B},
  journal={IEEE Trans. on Signal Proc.},
  volume={66},
  number={7},
  pages={1663--1675},
  year={2017},
  publisher={IEEE}
}

@article{conway1996,
    title={Packing lines, planes, etc., packing in Grassmannian spaces},
    author={Conway, John H. and Hardin, Ronald H. and Sloane, Neil J. A.},
    journal={Exper. Math.},
    volume={5},
    number={2},
    pages={139--159},
    year={1996}
}
